\documentclass[letterpaper, 10 pt, conference]{ieeeconf} 

\IEEEoverridecommandlockouts 
\usepackage{authblk}
\usepackage{comment}

\usepackage{graphicx}
\usepackage[linesnumbered,ruled,vlined]{algorithm2e}
\usepackage{xcolor}

\usepackage{amsmath}
\usepackage{arydshln}

\usepackage[export]{adjustbox}
\usepackage{enumerate}   
\usepackage{hhline}
\usepackage{eucal}
\usepackage{cite}
\usepackage{amssymb}
\usepackage{enumerate}
\usepackage{subfigure}
\usepackage{bm}
\usepackage{color}
\usepackage{multirow}
\newtheorem{thm}{Theorem}

\newcommand{\bJ}{\bm{J}}

\newcommand{\bT}{\bm{T}}
\newcommand{\bM}{\bm{M}}
\newcommand{\bC}{\bm{C}}

\newcommand{\bg}{\bm{g}}

\newcommand{\bxi}{\bm{\xi}}

\newcommand{\bq}{\bm{q}}

\newcommand{\br}{\bm{r}}
\newcommand{\bv}{\bm{v}}
\newcommand{\bV}{\bm{V}}

\newcommand{\bw}{\bm{w}}

\newcommand{\bu}{\bm{u}}

\newcommand{\by}{\bm{y}}

\newcommand{\btau}{\bm{\tau}}
\newcommand{\bzero}{\bm{0}}
\newcommand{\bI}{\bm{I}}
\newcommand{\bR}{\bm{R}}

\newcommand{\bLambda}{\bm{\Lambda}}
\newcommand{\bGamma}{\bm{\Gamma}}
\newcommand{\bzeta}{\bm{\zeta}}

\title{\LARGE \bf
Passivity-based Decentralized Control for Collaborative Grasping of Under-Actuated Aerial Manipulators
}

\author{Jinyeong Jeong and Min Jun Kim
\thanks{This work was supported by the National Research Foundation of Korea (NRF) grant funded by the Korea government (MSIT)  No. 2021R1C1C1005232, and No. 2021R1A4A3032834.}

\thanks{The authors are with Intelligent Robotic Systems Laboratory, Korea Advanced Institute of Science and Technology, Daejeon, Republic of Korea. {E-mail: {\tt\small \{first name.last name\}@kaist.ac.kr}}}
}

\begin{document}

\maketitle
\thispagestyle{empty}
\pagestyle{empty}

\begin{abstract}
This paper proposes a decentralized passive impedance control scheme for collaborative grasping using under-actuated aerial manipulators (AMs). The AM system is formulated, using a proper coordinate transformation, as an inertially decoupled dynamics with which a passivity-based control design is conducted. Since the interaction for grasping can be interpreted as a feedback interconnection of passive systems, an arbitrary number of AMs can be modularly combined, leading to a decentralized control scheme. Another interesting consequence of the passivity property is that the AMs automatically converge to a certain configuration to accomplish the grasping. Collaborative grasping using 10 AMs is presented in simulation.
\end{abstract}

\section{Introduction}
As multi-rotors can reach anywhere in 3D space with little restriction, they have been applied to applications such as monitoring, mapping, and surveillance. Apart from these, there are active missions that involve physical interaction such as contact inspection \cite{trujillo2019novel, tognon2019truly, bodie2020active} and slung load transportation \cite{bernard2011autonomous, bernard2010load}. Indeed, a number of studies have proposed the design and control of aerial manipulators (AMs) \cite{korpela2012mm, orsag2013modeling, fumagalli2016mechatronic, sarkisov2019development, pounds2011yale}.

To accomplish active tasks, in this paper, we focus on AMs in which an $n$ degree-of-freedom (DOF) non-redundant robotic manipulator is mounted on an under-actuated multi-rotor. As the payload of an AM is typically limited, lightweight designs of manipulators have been proposed \cite{suarez2018design, bellicoso2015design}, however usually at the cost of reduced control performance.
To overcome the limited payload, this paper presents a control strategy that enables the collaboration of multiple AMs, within a scenario of collaborative grasping. In particular, we accomplish the scenario by applying a passive impedance controller to each AM in a decentralized manner.

However, dynamic coupling between an under-actuated floating base and a robotic manipulator makes the control problem complicated even for a single AM. To reduce such complexity, some studies developed fully-actuated multi-rotors by tilting propellers in a proper way \cite{tognon2019truly, trujillo2019novel}. Since the propellers are not collinear anymore, however, full actuation can be achieved only with reduced payload capacity.

With an under-actuated multi-rotor, several studies have proposed the design and control of AMs to accomplish compliant interaction. \cite{forte2012impedance, fumagalli2012modeling, lippiello2012exploiting} achieved stable physical interaction with mechanical elements that introduce compliance to the end-effector. \cite{acosta2014robust, kim2018passive, yuksel2019aerial} proposed an energy-based control method, which is essentially similar to the passivity-based control approaches. When a multi-DOF manipulator is used, \cite{yang2014dynamics, garofalo2018task} have shown that a proper coordinate transformation leads to an inertially decoupled dynamics with which control design becomes much more convenient. Inspired by prior work, we will utilize a passivity-based control method with the inertially decoupled AM dynamics.

Control problem becomes even more complicated for the collaborative grasping scenario due to the coupled dynamics between multiple AMs. To simplify the problem, some studies assume that an object can be rigidly attached to the end-effectors using, for example, magnetics \cite{jimenez2022precise, tagliabue2019robust}. In this setup, multi-agent coordination is of interest rather than the realization of stable grasping. To tackle the grasping problem explicitly, \cite{mohammadi2016cooperative, yang2015hierarchical} posed an optimization problem, and \cite{caccavale2015cooperative} designed an internal impedance controller using a grasp matrix. However, since these methods require a centralized controller, scalability is questionable when many AMs are involved in the mission.

\begin{figure} [!t]
    \centering
	{\includegraphics[scale=0.7]{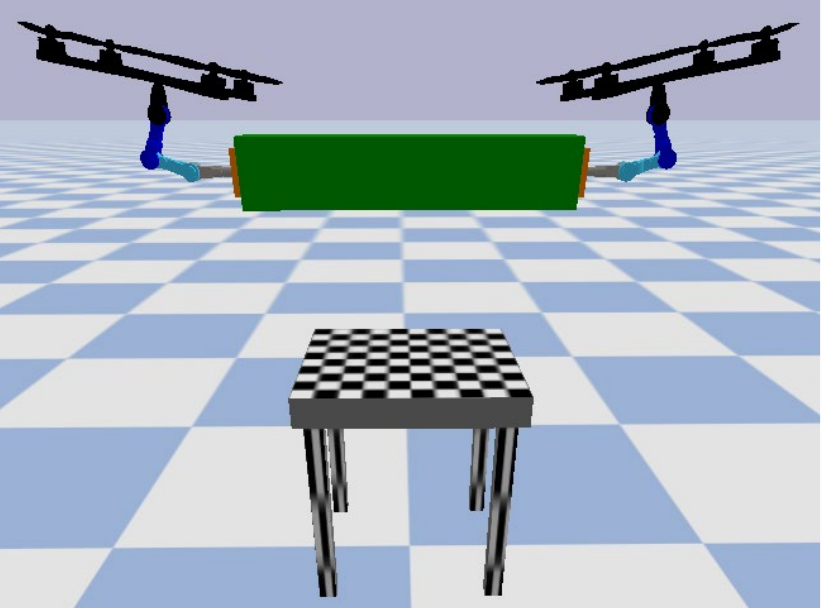}} 
    	\caption{This work aims at accomplishing collaborative grasping using multiple AMs with a decentralized control scheme. The passivity-based impedance control makes AMs converge to certain configurations at which they can balance each other.}
	\label{fig:collaborative_grasping}
\end{figure}

In this paper, to enable a decentralized control scheme, we use the passivity theory that leads to modularity due to the well-known fact that passivity is preserved under a feedback interconnection of passive systems \cite{zelazo2015decentralized, kim2019passivity, kim2021passive}. To this end, we employ a coordinate transformation \cite{garofalo2018task} that results in an inertially decoupled AM dynamics to which we can apply a passive impedance control method. As the passivity property provides modularity, the proposed method is scalable in the sense that many AMs can collaboratively grasp an object in a decentralized manner. More specifically, the states of other AMs do not appear in the control law.

Technical contributions of this paper can be summarized as follows. At the pre-grasping phase (i.e. at the free-flight phase), uniform asymptotic stability is achieved for each AM. Under the interaction, we show that the input-output (I/O) pair of an AM is output strictly passive. Consequently, collaborative grasping is accomplished stably by successively applying feedback interconnections that preserve passivity. Using the proposed control strategy, each AM converges to a certain configuration at which multiple AMs balance each other, achieving a stable hovering (see Fig. \ref{fig:collaborative_grasping}).

This paper is organized as follows. Section \ref{sec:system_modeling} presents an inertially decoupled AM dynamics. In Section \ref{sec:control_framework}, a decentralized passive impedance control scheme is proposed and convergence under collaborative grasping is shown through passivity analysis. Section \ref{sec:simulation} validates the proposed control scheme via simulations and Section \ref{sec:conclusion} concludes the paper.

\section{System Modeling}
\label{sec:system_modeling}

\subsection{Dynamic equations of aerial manipulator}
Let us consider an under-actuated aerial vehicle equipped with an $n$ DOF non-redundant robotic manipulator as shown in Fig. \ref{fig:am_modeling}. $\{w\}$, $\{b\}$, and $\{e\}$ represent the world frame, the floating base frame, and the end-effector frame. $\{b\}$ is located at the center of mass (CoM) of the aerial vehicle. 
\begin{figure} [!t]
    \centering
	{\includegraphics[scale=0.7]{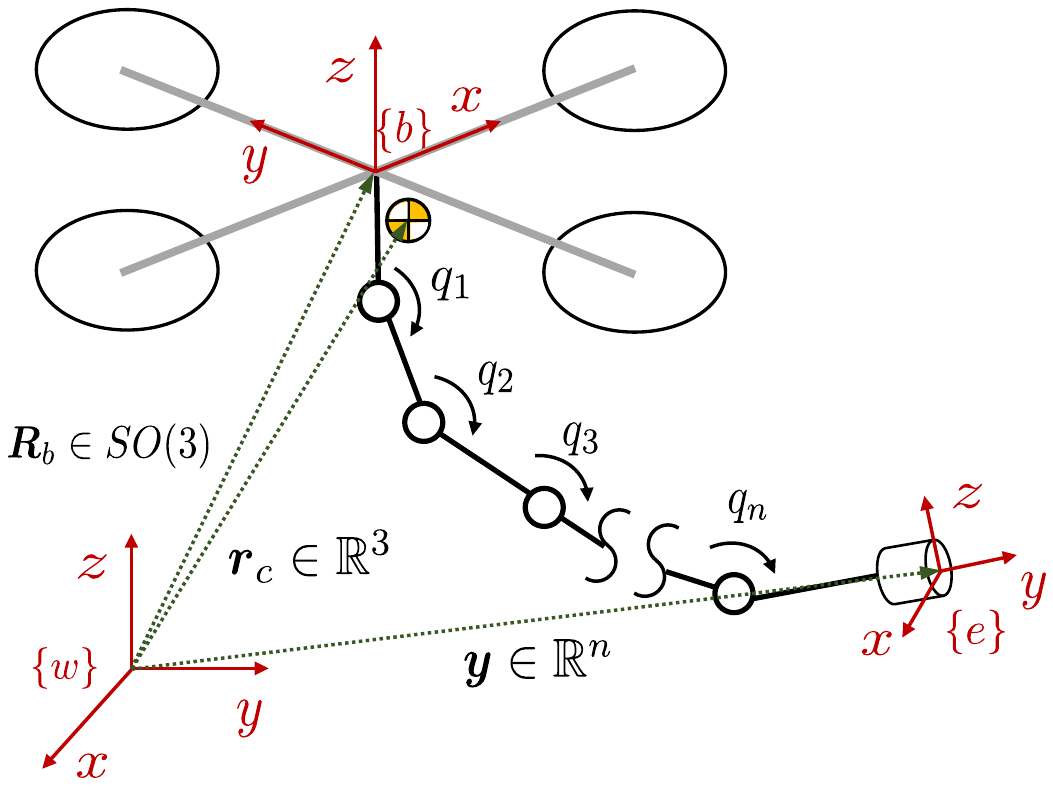}} 
    	\caption{An under-actuated aerial vehicle equipped with an $n$ DOF non-redundant robotic manipulator. }
	\label{fig:am_modeling}
\end{figure}
Assume that all external forces are applied only at the end-effector by the environmental interaction. Then, the dynamic equations of the system can be written as
\begin{align}
\label{eq:dynamics}
\bM(\bq)\dot{\bV}+\bC(\bq, \bV)\bV+\bg(\bR_b,\bq)=\btau+\bJ_e^T\bm{F}_e,
\end{align}
with $\bV=[\bv_b^T \;\bw_b^T \; \dot{\bq}^T]^T\in\mathbb{R}^{n+6}$. $\bv_b$ and $\bw_b$ are the body linear and angular velocities of the floating base. ${\bq\in\mathbb{R}^n}$ is the joint configuration of the manipulator, and ${\bR_b\in SO(3)}$ represents the floating base orientation. $\bM, \bC$, and $\bg$ are the inertia matrix, Coriolis/centrifugal matrix, and gravity vector, respectively.

${\btau=[0\;0\;f\;\btau_b^T \;\btau_q^T]^T\in\mathbb{R}^{n+6}}$ is the generalized force, where $f$ is the total thrust of the under-actuated aerial vehicle, ${\btau_b\in\mathbb{R}^3}$ is the torque of the floating base, and $\btau_q\in\mathbb{R}^n$ is the joint torque of the manipulator. Note that the system is under-actuated because the thrust can only be generated along the body $z$-axis. Next, ${\bm{F}_e=[\bm{f}_e^T\; \btau_e^T]^T\in\mathbb{R}^6}$ is the end-effector's body wrench, where $\bm{f}_e$ and $\btau_e$ represent force and moment, respectively. $\bJ_e\in\mathbb{R}^{6\times (n+6)}$ is the end-effector's body Jacobian that maps $\bV$ to the end-effector's body twist; i.e., $\bV_e=[\bv_e^T \;\bw_e^T]^T=\bJ_e\bV$. Here, $\bv_e$ and $\bw_e$ are the body linear and angular velocities of $\{e\}$.

\subsection{Inertially decoupling coordinate transformation}

Motivated by \cite{garofalo2018task}, we employ the following coordinate transformation which maps $\bV$ to a new coordinate ${\bxi\in\mathbb{R}^{n+6}}$.
\begin{align}
\label{eq:new_coord}
\bm{\bxi}\triangleq&\left[
\begin{array}{c}
\dot{\br}_{c}\\
\bw_b\\
\bm{\rho}
\end{array}
\right]
=\left[
\begin{array}{c}
\bJ_{c}\\
\bJ_{w_b}\\
\bm{N}
\end{array}
\right]\bV\\
=&\underbrace{\left[
\begin{array}{ccc}
\bR_b &-\hat{\br}_{bc}\bR_b &\partial \br_{bc}/{\partial \bq}\\
\bzero & \bI & \bzero\\
\bzero &\bm{N}_1 &  \bI
\end{array}
\right]}_{=\bm{T}}\bV,
\label{eq:simplified_N}
\end{align}
where $\br_{c}$ is the CoM of the aerial manipulator represented in $\{w\}$ (see Fig. \ref{fig:am_modeling}). ${\br}_{bc}$ is the vector from $\{b\}$ to the CoM of the aerial manipulator, expressed in $\{w\}$. The hat operator maps $\mathbb{R}^3$ to $\mathfrak{so}(3)$, i.e.,
\begin{align}
\label{eq:hat}
\widehat{\bm{r}}=\left[
\begin{array}{ccc}
0 & -r_z & r_y \\ 
r_z & 0 & -r_x \\ 
-r_y & r_x & 0 
\end{array}
\right] \in \mathfrak{so}(3),   
\end{align}
where $\bm{r}=[r_x, \; r_y, \; r_z]^T\in\mathbb{R}^3$. $\bm{\rho}\in\mathbb{R}^n$ is the nullspace velocity and $\bm{N}$ is a dynamically consistent nullspace projector defined by ${\bm{N}=\bm{(ZMZ^T)^{-1}ZM}}$ \cite{ott2008cartesian}. Here, ${\bm{Z}=[-(\partial \br_{bc}/{\partial \bq})^T\bR_b\;\;\bzero_{n\times3}\;\bI]}\in\mathbb{R}^{n\times(n+6)}$ is a nullspace base matrix, satisfying $\bm{Z}[\bJ_{c}^T\; \bJ_{w_b}^T]=\bzero$.

Note that the explicit computation of $\bm{N}$ results in ${\bm{N}=[\bzero_{n\times3} \; \bm{N}_1 \; \bI]}$ (see (\ref{eq:simplified_N})) using ${{\bM_{mt}}=m(\partial \br_{bc}/{\partial \bq})^T\bR_b\in\mathbb{R}^{n\times3}}$, where ${\bM_{mt}}$ is the inertia coupling matrix between the manipulator and the floating base translation, and $m$ is the total mass. Moreover, $\bT^{-T}$ has the following structure. This can be easily validated by calculating it by hand.
\begin{align}
\label{eq:T_structure}
\bT^{-T}|_{\text{first\; $3 \times (n+6)$}}=[\bR_b \; \bzero_{3\times3}\;\bzero_{3\times n}].
\end{align}

Employing the coordinate transformation $\bm{T}$, interestingly, the transformed inertia matrix $(\bm{TM}^{-1}\bm{T}^T)^{-1}$ becomes block-diagonal. Using $\bxi=\bm{T}\bV$ in (\ref{eq:new_coord}), the dynamic equations (\ref{eq:dynamics}) can be rewritten in the coordinate $\bxi$ as follows: 
\begin{align}
\label{eq:new_dynamics}
&\underbrace{
\left[
\begin{array}{ccc}
 m\bI & \bzero &\bzero\\
 \bzero & \bLambda_{w_b} & \bzero\\
 \bzero &\bzero&\bLambda_{\rho}
 \end{array}
\right]}_{=\bLambda_\xi}\dot{\bxi}+
\underbrace{
\left[
\begin{array}{ccc}
 \bzero & \bzero &\bzero\\
 \bzero & \bGamma_{w_b} & \bGamma_{w_b, \rho}\\
 \bzero &-\bGamma_{w_b, \rho}^T & \bGamma_{\rho}
 \end{array}
\right]}_{=\bGamma_\xi}
\bxi\nonumber\\
& +\underbrace{
\left[
\begin{array}{c}
 mg\bm{e}_3\\
 \bzero\\
 \bzero
 \end{array}
 \right]}_{=\bzeta_\xi}=\underbrace{\left[
\begin{array}{c}
u_1\bR_b  \bm{e}_3\\
\bm{u}_2\\
\bm{u}_3\\
 \end{array}
 \right]}_{=\bm{u}}+
\underbrace{\left[
\begin{array}{c}
\bm{F}_{r_c}\\
\bm{F}_{w_b}\\
\bm{F}_{\rho}\\
 \end{array}
 \right],}_{=\bm{F}_\xi}
\end{align}
with the transformed inertia matrix $\bm{\Lambda}_\xi$, Coriolis/centrifugal matrix $\bGamma_\xi$, gravity vector $\bzeta_\xi$, the gravitational acceleration $g$, and $\bm{e}_3=[0\; 0\; 1]^T$. 

$\bu$ and $\bm{F}_\xi$ are the transformed generalized force and external forces given by ${\bu=\bT^{-T}\btau}$ and ${\bm{F}_\xi=\bT^{-T}\bJ_e^T\bm{F}_e}$. By calculating ${\bu=\bT^{-T}\btau}$, in fact, one can easily check that $u_1=f$. Note also that the explicit computation of the first three elements of $\bm{F}_\xi=\bT^{-T}\bJ_e^T\bm{F}_e$ leads to ${\bm{F}_{r_c}=\bR_e\bm{f}_e}$, where $\bR_e$ is the orientation of $\{e\}$ relative to $\{w\}$. 

The first, second, and third rows of (\ref{eq:new_dynamics}) correspond to the CoM dynamics, the floating base orientation dynamics, and the nullspace dynamics, respectively. We remark that the first and the second rows of (\ref{eq:new_dynamics}) have a similar structure to the standard under-actuated multi-rotor dynamics.

\section{Control Framework}
\label{sec:control_framework}

\subsection{Control strategy}

This paper aims at grasping an object using an arbitrary number of AMs which are under-actuated. Since a dynamic equation of one AM is given by (\ref{eq:new_dynamics}), that of multiple AMs coupled via $\bm{F}_e$ becomes very complicated to handle. To overcome this, we propose a decentralized control scheme in which a passive impedance control is applied to each AM. Since a feedback interconnection of passive systems preserves passivity \cite{khalil2002nonlinear}, the proposed control scheme achieves modularity that ensures stable collaborative grasping with any number of AMs.

Recall that the first and second rows of (\ref{eq:new_dynamics}) have a similar structure to the standard multi-rotor dynamics. Thus, we utilize a geometric tracking control, which is a well-established approach for multi-rotors \cite{lee2010geometric}, with little modification. On the third row of (\ref{eq:new_dynamics}), we apply a technique similar to the output feedback linearization, to realize a passive impedance behavior on the manipulator's end-effector.

From the theoretical point of view, we show (i) uniform asymptotic stability for free-flight, (ii) passivity property of a single AM, and (iii) convergence under collaborative grasping. The first one is relevant to the pre-grasping phase in which AMs fly independently, and the others come into play when the grasping occurs.

\label{sec:control_goal}

\subsection{Control design}

\subsubsection{CoM control}

Consider the control law given by 
\begin{align}
\label{eq:Com_controller}
u_1=(\underbrace{m\ddot{\br}_{c, d}+mg\bm{e}_3-\bm{K}_t\tilde{\br}_{c}-\bm{D}_t\dot{\tilde{\br}}_{c}-\bm{F}_{r_c}}_{=\bm{f}_d})\cdot(\bR_b\bm{e}_3),
\end{align}
with positive definite gain matrices $\bm{K}_t$ and $\bm{D}_t$. ${\tilde{\br}_{c}=\br_{c}-\br_{c, d}}$ is the CoM position error where  $\br_{c, d}$ is the desired trajectory of the CoM. $\bm{f}_d$ is the PD control with an acceleration feed-forward term and the external force compensation including the gravity. By compensating for the external force, the CoM dynamics is not affected by the environmental interaction.

$u_1$ is designed as the projection of $\bm{f}_d$ onto the floating base z-axis. Therefore, to stabilize the CoM dynamics, $\bm{f}_d$ has to be directed toward the body z-axis of the floating base as much as possible. To this end, the desired floating base rotation can be defined as follows\cite{lee2010geometric}:
\begin{align}
\label{eq:Rd}
\bR_{b, d}=[\bm{b}_{2d}\times\bm{b}_{3d}\;\; \bm{b}_{2d}\;\; \bm{b}_{3d}],
\end{align}
where $\bm{b}_{3d}=\bm{f}_d/\|\bm{f}_d\|$ and $\bm{b}_{2d}=(\bm{b}_{3d}\times\bm{b}_{1d})/\|\bm{b}_{3d}\times\bm{b}_{1d}\|$ with the desired heading direction $\bm{b}_{1d}$ of the floating base. Assume that $\|\bm{f}_d\| \neq 0$.
\label{sec:com_control}
\subsubsection{Floating base orientation control}
To achieve ${\bR_b\rightarrow\bR_{b, d}}$ (\ref{eq:Rd}), $\bu_2$ is designed as follows:
\begin{align} 
\label{eq:ori_controller}
\bu_2=\bLambda_{w_b}[\dot{\bw}_{b, d}-&\bLambda^{-1}_{w_b, d}(k_R\bm{e}_R+k_w\bm{e}_w)]\nonumber\\&+\bGamma_{w_b}\bw_b+\bGamma_{w_b, \rho}\bm{\rho}-\bm{F}_{w_b},
\end{align}
where $\bLambda_{w_b, d}>\bzero$ is the desired inertia, and $k_w, k_R$ are positive scalar gains that respectively represent the desired damping and stiffness values. The orientation error $\bm{e}_R$ and the angular velocity error $\bm{e}_w$ are defined by
\begin{align}
\label{eq:ori_error}
\bm{e}_R=\frac{1}{2}({\bR^T_{b, d}}\bR_b-\bR_b^T\bR_{b, d})^\vee,\\
\label{eq:ori_vel_error}
\bm{e}_w=\bw_b-\underbrace{\bR_b^T\bR_{b, d}({\bR^T_{b, d}}\dot{\bR}_{b,d})^\vee}_{=\bw_{b,d}},
\end{align}
where the vee operator ${(\cdot)^\vee: \mathfrak{so}(3) \mapsto \mathbb{R}^3}$ is the inverse of the hat operator.

This control law leads to the following mass-spring-damper closed-loop system.
\begin{align}
\label{eq:closed_loop}
\bLambda_{w_b, d}\dot{\bm{e}_w}+k_w\bm{e}_w+k_R\bm{e}_R=\bzero.
\end{align}
\label{sec:ori_control}

\subsubsection{End-effector impedance control}
Let us define the task variable $\by\in\mathbb{R}^n$, which represents the end-effector configuration in the task space with respect to $\{w\}$; see Fig. \ref{fig:am_modeling}. The goal is to realize a mass-spring-damper impedance behavior with the variable $\tilde{\by}=\by-\by_d$, where $\by_d$ represents the desired trajectory of $\by$.

To begin with, notice that there exists a Jacobian matrix ${\bJ}$ that relates $\bxi$ to ${\dot{\by}}$:
\begin{align}
\label{eq:jacobian_xi}
\dot{\by}=\bm{J}\bxi,
\end{align}
where ${\bJ\in\mathbb{R}^{n\times(n+6)}}$. Using (\ref{eq:new_dynamics}), $\ddot{\tilde{\by}}$ can be represented as
\begin{align}
\label{eq:tilde_y_ddot}
\ddot{\tilde{\by}}=\ddot{{\by}}-\ddot{\by}_d=\bJ\underbrace{\bLambda_\xi^{-1}(\bu+\bm{F}_\xi-\bGamma_{\xi} \bxi-\bzeta_\xi)}_{=\dot{\bxi}}+\dot{\bJ}\bxi-\ddot{\by}_d.
\end{align}
Pre-multiplying both sides of (\ref{eq:tilde_y_ddot}) by ${\bLambda_y=(\bJ\bLambda_\xi^{-1}\bJ^T)^{-1}}$, we obtain

\begin{align}
\label{eq:m_tilde_y_ddot_}
\bLambda_y\ddot{\tilde{\by}}=\bm{F}_y+\underbrace{(\bJ^{\Lambda+})^T(\bu-\bGamma_{\xi} \bxi-\bzeta_\xi)+\bLambda_y(\dot{\bJ}\bxi-\ddot{\by}_d)}_{=\bar{\bu}},
\end{align}
where ${\bJ^{\Lambda+}\triangleq \bLambda_\xi^{-1}\bJ^T(\bJ\bLambda_\xi^{-1}\bJ^T)^{-1}}$ is a right inverse of $\bJ$; i.e., $\bJ\bJ^{\Lambda+}=\bI_{n\times n}$.  ${\bm{F}_y=(\bJ^{\Lambda+})^T\bm{F}_\xi}$ is a generalized force in $\dot{\tilde{y}}$-coordinate.

The control goal is to realize the following impedance behavior using the control input $\bu_3$. 
\begin{align}
\label{eq:desired_task_behavior}
\bLambda_{y}\ddot{\tilde{\by}}+\bm{D}_y\dot{\tilde{\by}}+\bm{K}_y\tilde{\by}=\bm{F}_{y},
\end{align}
where ${\bm{K}_y, \bm{D}_y >\bzero}$ represent the stiffness, and damping gains, respectively. In other words, we need to design $\bu_3$ in such a way that ${\bar{\bu}=-(\bm{D}_y\dot{\tilde{\by}}+\bm{K}_y\tilde{\by})}$. This can be done by solving
\begin{align}
\label{eq:u_bar}
\left[
\begin{array}{ccc}
\bar{\bJ}_{1} & \bar{\bJ}_{2} & \bar{\bJ}_{3}
 \end{array}
\right]\left[
\begin{array}{c}
f\bR_b\bm{e}_3-mg\bm{e}_3\nonumber\\
\bu_2-\bGamma_{w_b}\bw_b-\bGamma_{w_b, \rho}\bm{\rho}\\
\bu_3+\bGamma_{w_b, \rho}^T\bw_b-\bGamma_{\rho}\bm{\rho}
\end{array}
\right]\nonumber\\
=-(\bLambda_y(\dot{\bJ}\bxi-\ddot{\by}_d)+\bm{D}_y\dot{\tilde{\by}}+\bm{K}_y\tilde{\by}),
\end{align}
where $(\bJ^{\Lambda+})^T=[\bar{\bJ}_{1}\; \bar{\bJ}_{2}\; \bar{\bJ}_{3}]\in\mathbb{R}^{n\times(n+6)}$. Assuming that $\bar{\bJ}_{3}$ is invertible,\footnote{With the aid of symbolic calculation, we confirmed that $\bar{\bJ}_{3}$ is almost always invertible unless the manipulator is in a singular configuration.} 
the resulting $\bu_3$ is
\begin{align}
\label{eq:u3}
&\bu_3=-\bar{\bJ}_{3}^{-1}(
\left[
\begin{array}{cc}
\bar{\bJ}_{1} & \bar{\bJ}_{2}
 \end{array}
\right]\left[
\begin{array}{c}
f\bR_b\bm{e}_3-mg\bm{e}_3\nonumber\\
\bu_2-\bGamma_{w_b}\bw_b-\bGamma_{w_b, \rho}\bm{\rho}
\end{array}
\right]
\\&+\bLambda_y(\dot{\bJ}\bxi-\ddot{\by}_d)+\bm{D}_y\dot{\tilde{\by}}+\bm{K}_y\tilde{\by})+\bGamma_{\rho}\bm{\rho}-\bGamma_{w_b, \rho}^T\bw_b.
\end{align}
The control inputs $u_1$, $\bu_2$, and $\bu_3$ can be transformed into ${\btau=[0\;0\;f\;\btau_b^T \;\btau_q^T]^T\in\mathbb{R}^{n+6}}$ using $\btau=\bT^T\bu$.

\subsection{Controller analysis}

This section presents theoretical justifications for the proposed controller. After presenting the stability in free-flight, stable collaboration will be followed based on passivity analysis. Throughout this section, states $\bm{x}$ are omitted in Lyapunov and storage functions for the sake of simplicity; e.g., $V(\bm{x})$ will be expressed as $V$. The following theorem shows uniform asymptotic stability, which is relevant to the free-flight (i.e. pre-grasping) phase.
\begin{thm}[Uniform asymptotic stability of free-flight]
	\label{thm:stability}
	Assume that $\bm{F}_e=\bzero$ and $\|m\ddot{\br}_{c, d}+mg\bm{e}_3\|<B_1$. With the control law (\ref{eq:Com_controller}), (\ref{eq:ori_controller}), and (\ref{eq:u3}), an AM system has a uniformly asymptotically stable equilibrium point $\tilde{\br}_c=\bzero, \tilde{\bm{e}}_R=\bzero$, and $\tilde{\by}=\bzero$ with zero derivatives.
\end{thm}
\begin{proof}
    The candidate Lyapunov function is given by
	\begin{align}
    \label{eq:lyapunov}
    V&=V_1+\frac{1}{2}\dot{\tilde{\by}}^T\bLambda_y\dot{\tilde{\by}}+\frac{1}{2}\tilde{\by}^T\bm{K}_y\tilde{\by},\\
    V_1&=\frac{1}{2}m\|\dot{\tilde{\br}}_c\|^2+\frac{1}{2}\tilde{\br}_c^T\bm{K}_t\tilde{\br}_c+c_1\tilde{\br}_c^T\dot{\tilde{\br}}_c\nonumber\\&+\frac{1}{2}\bm{e}_w^T\bLambda_{w_b,d}\bm{e}_w+\frac{1}{2}K_rtr(\bI-\bR_{b,d}^T\bR_b)+c_2\bm{e}_R^T\bm{e}_w.
    \label{eq:lyapunov_aerial_vehicle}
    \end{align}
    In fact, \cite{lee2010geometric} already proved that there exist positive definite functions $W_1, W_2$, and $W_3$ such that ${W_1\leq V_1\leq W_2}$ and ${\dot{V}_1\leq-W_3}$. In addition, since all eigenvalues of $\bLambda_y$ and $\bm{K}_y$ are positive, there exist positive definite functions $\bar{W}_1$, $\bar{W}_2$ such that ${\bar{W}_1\leq V\leq \bar{W}_2}$. Therefore, using (\ref{eq:desired_task_behavior}), the time derivative of $V$ satisfies ${\dot{V}\leq -W_3-\dot{\tilde{\by}}^T\bm{D}_y\dot{\tilde{\by}}}$, meaning that $\dot{V}$ is negative semi-definite. Hence, the uniform stability is achieved and all states are bounded. 
    
    Note that $V$ has a finite limit because $\dot{V}\leq0$ and $V$ is lower bounded. Moreover, as $\ddot{V}$ exists and is bounded, $\dot{V}$ is uniformly continuous in time. By applying Barbalat's lemma \cite{slotine1991applied}, ${\dot{V}\rightarrow 0 \Rightarrow (\tilde{\br}_c, \dot{\tilde{\br}}_c, \bm{e}_R, \bm{e}_w, \dot{\tilde{\by}})\rightarrow\bzero}$, resulting in ${\tilde{\by}\rightarrow\tilde{\by}^*}$ due to the convergence of $V$. To show that ${\tilde{\by}^*=\bzero}$ by contradiction, suppose that ${\tilde{\by}^*\neq\bzero}$. This leads to ${\ddot{\tilde{\by}}\rightarrow\ddot{\tilde{\by}}^*\neq\bzero}$ 
   (see (\ref{eq:desired_task_behavior})), which is contradiction to $\dot{\tilde{\by}}\rightarrow\bzero$. As a result, ${\tilde{\br}_c=\tilde{\bm{e}}_R=\tilde{\by}=\bzero}$ with zero derivatives is a uniformly asymptotically stable equilibrium point.
\end{proof}

The following theorem states the passivity property of a single AM, which will be applied to the multi-AM scenario.

\begin{thm}[Passivity under interaction]
	\label{thm:passivity}
	Assume that ${\|m\ddot{\br}_{c, d}+mg\bm{e}_3-\bm{F}_{r_c}\|<B_2}$, and let $\dot{\by}_d=0$.  
    With the proposed controller (\ref{eq:Com_controller}), (\ref{eq:ori_controller}), and (\ref{eq:u3}), there exists a storage function $S_{AM}\geq0$ such that
	\begin{align}
	    \dot{S}_{AM} \leq -\lambda_{m}(\bm{D}_y)||\dot{\by}||^2 + \dot{\by}^T\bm{F}_y,
	    \label{eq:storage_diff_single_am}
	\end{align}
	where $\lambda_{m}(\cdot)$ denotes the minimum eigenvalue of $(\cdot)$. (\ref{eq:storage_diff_single_am}) means that the input-output (I/O) pair $(\bm{F}_y, \dot{\by})$ of an AM is output strictly passive. 
\end{thm}

\begin{proof}
    Lyapunov function $V$ used in Theorem \ref{thm:stability} will be used as a storage function, i.e., $S_{AM}=V$. With the assumption ${\|m\ddot{\br}_{c, d}+mg\bm{e}_3-\bm{F}_{r_c}\|<B_2}$, the procedure of Theorem \ref{thm:stability} can be followed almost exactly the same, and consequently, the time derivative of $S_{AM}$ is ${\dot{S}_{AM}\leq -W_3-\dot{{\by}}^T \bm{D}_y \dot{{\by}}+\dot{\by}^T\bm{F}_y}$. Here, although $W_3$ is different from Theorem \ref{thm:stability} due to the existence of $\bm{F}_{r_c}$, it is still a positive definite function. Hence (\ref{eq:storage_diff_single_am}) can be concluded. 
\end{proof}

\begin{figure} [!t]
    \centering
	{\includegraphics[scale=0.7]{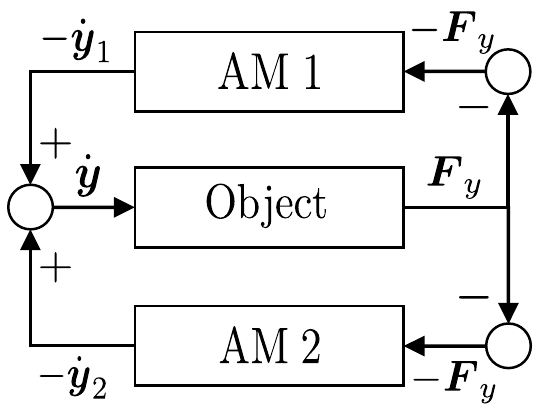}} 
    	\caption{Collaborative grasping using two AMs can be interpreted as a feedback interconnection of passive systems. }
	\label{fig:two_collaboration}
\end{figure}

\begin{figure*} [!t]
    \centering
	\subfigure[]
	{\includegraphics[scale=0.64 ]{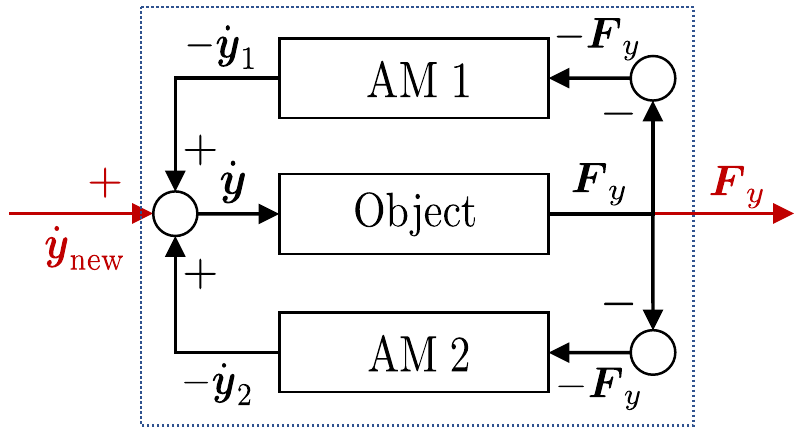}}
	\centering
	\subfigure[]
	{\includegraphics[scale=0.54]{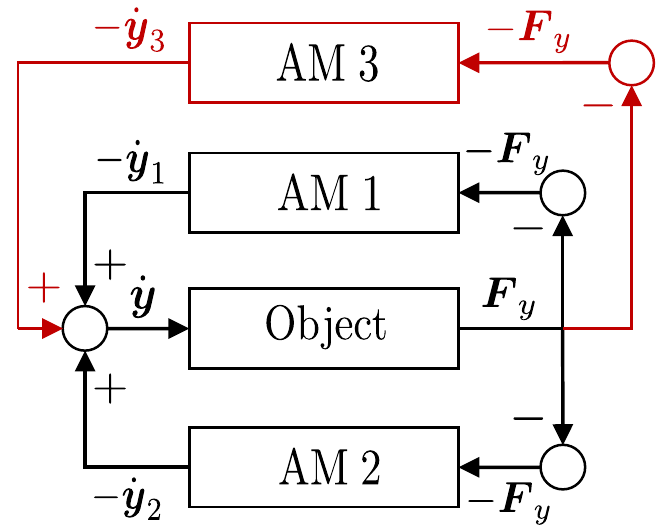}}
	\centering
	\subfigure[]
	{\includegraphics[scale=0.60]{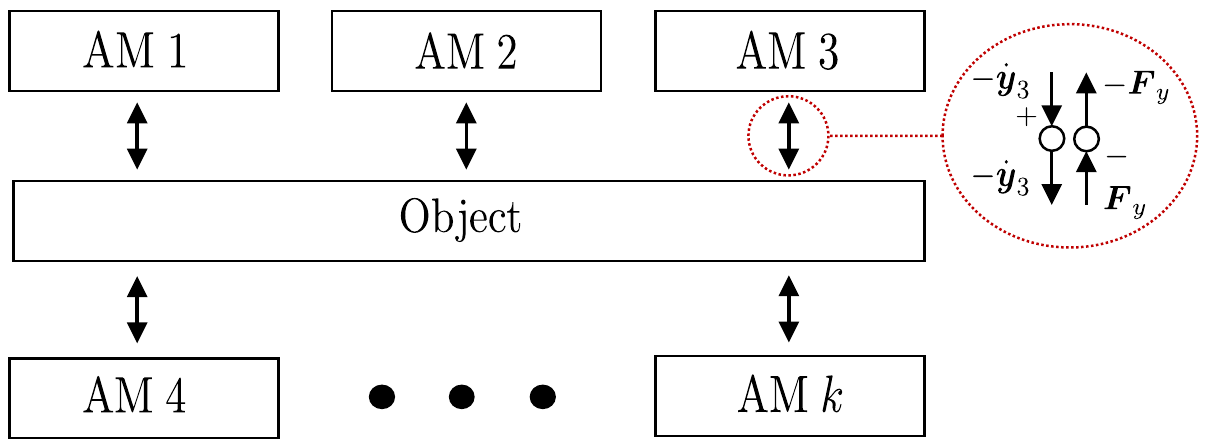}}
	\centering
	\caption{(a) The new I/O pair $(\dot{\by}_{new}, \bm{F}_y)$ of the entire system is passive. (b) The new I/O can be interconnected with `AM 3' by feedback, which preserves passivity. (c) Collaborative grasping with multiple AMs can be achieved by repeating the feedback interconnection. 
	}
	\label{fig:feedback_interconnection}
 	\vspace{-5mm}
\end{figure*}

To analyze collaborative grasping, consider an object whose dynamics is governed by $\bm{F}_y=\bm{M}\ddot{\by}$.\footnote{For simplicity, we omitted the gravity force which does not influence the passivity analysis. The gravity force is compensated by the friction caused by the internal force produced by impedance controllers of AMs.} Then passivity of an I/O pair ($\dot{\by}$, $\bm{F}_y$) is trivial by using ${S_{obj}=\frac{1}{2}\dot{\by}^T\bm{M}\dot{\by}}$ as a storage function, because 
\begin{align}
    \dot{S}_{obj}=\bm{F}_y^T \dot{\by}.
    \label{eq:S_obj_deriv}
\end{align}
When an object is grasped by two AMs, a block diagram can be represented by a feedback interconnection as shown in Fig. \ref{fig:two_collaboration}. An interpretation is straightforward: input to the object is the net velocity of AMs' end-effectors, while the output is the inertial force that acts as an external force to the AMs. The following theorem states that two AMs converge to a certain value.

\begin{thm}[Convergence under collaborative grasping]
	\label{thm:stability_grasping}
	Suppose that two AMs are grasping an object, which is modeled as a single mass, using the control law (\ref{eq:Com_controller}), (\ref{eq:ori_controller}), and (\ref{eq:u3}). Then the AMs converge to a certain configuration.
\end{thm}

\begin{proof}
    Since the convergence of $\dot{\bm{r}}_c$- and $\bm{w}_b$-dynamics can be achieved regardless of interaction, it is sufficient to consider $\by$-dynamics in the analysis. Let ${S_{tot}=S_{AM1}+S_{AM2}+S_{obj}}$ be the storage function. The time derivative yields, using (\ref{eq:storage_diff_single_am}) and (\ref{eq:S_obj_deriv}), 
    \begin{align}
        \nonumber \dot{S}_{tot} \leq & -\lambda_{m}(\bm{D}_y)(||\dot{\by}_1||^2+||\dot{\by}_2||^2)  \\
        \nonumber &+ {\underbrace{(-\dot{\by}_1 - \dot{\by}_2)}_{\dot{\by} \text{ in Fig. \ref{fig:two_collaboration}}}}^T\bm{F}_y + \dot{\by_1}^T \bm{F}_y + \dot{\by_2}^T \bm{F}_y \\
        &= -\lambda_{m}(\bm{D}_y)||\dot{\bar{\by}}||^2,
    \end{align}
    where $\bar{\by} = [\by_1^T, \; \by_2^T]^T$. Integrating back,
    \begin{align}
        -\int \lambda_{m}(\bm{D}_y)||\dot{\bar{\by}}||^2 \geq S_{tot}(t) - S_{tot}(0) \geq -S_{tot}(0).
    \end{align}
    It follows that $\int ||\dot{\bar{\by}}||^2 \leq c^2$, for some constant $c$. This upper bound implies that $\dot{\bar{\by}}$ converges to $\bzero$, which again implies that $\bar{\by}$ converges to a certain value.\footnote{To be precise, uniform continuity of $\dot{\bar{\by}}$ is required, which is guaranteed by boundedness of $\ddot{\bar{\by}}$.}
\end{proof}

\subsection{Controller discussion}

\subsubsection{Extension to many AMs}
\label{sec:discussion_scalability}
Collaborative grasping using two AMs is analyzed in Theorem \ref{thm:stability_grasping}. In fact, this can be extended to an arbitrary number of AMs. Using $S_{tot}$ as a storage function, it can be easily show that the new I/O pair $(\dot{\by}_{new},\bm{F}_y)$ in Fig. \ref{fig:feedback_interconnection}(a) is passive. Then, this I/O port can be again interconnected with `AM 3' by feedback that preserves passivity (see Fig. \ref{fig:feedback_interconnection}(b)). By repeating this process $k$ times, Theorem \ref{thm:stability_grasping} can be extended to many AMs (see Fig. \ref{fig:feedback_interconnection}(c)). Note that, with this process, the states of other AMs do not appear in the control law (\ref{eq:Com_controller}), (\ref{eq:ori_controller}), and (\ref{eq:u3}), resulting in a decentralized control scheme.

An interesting consequence of Theorem \ref{thm:stability_grasping} is that, while hovering, the AMs converge to a certain configuration at which they can balance each other (recall Fig. \ref{fig:collaborative_grasping}). This is practically appealing, because, it easily becomes impractical to compute the desired orientation and the thrust force as the number of AMs increases. Simulations in Section \ref{sec:simulation} will demonstrate this property.

 \subsubsection{Force/Torque sensor measurement}
\label{sec:discussion_ft}

One drawback of the proposed control would be the necessity of a force/torque sensor at the end-effector because force compensation terms appear in (\ref{eq:Com_controller}) and (\ref{eq:ori_controller}). We claim that the force compensation can be neglected in practice.

To theoretically support this claim, we remark that the Lyapunov function $V_1$ used in Theorem \ref{thm:stability} can be used as an input-to-state stability (ISS)-Lyapunov function. The ISS property indicates that the system is robust against input ($\bm{F}_e$ for our case), in the sense that the states will be bounded for bounded input. Moreover, physically speaking, since (\ref{eq:Com_controller}) and (\ref{eq:ori_controller}) define a mass-spring-damper behavior, the control can balance external forces without explicit compensation.

\subsubsection{Force closure}

Despite the fact the force closure condition is crucial to guarantee stable grasping \cite{murray2017mathematical}, some studies proposed an alternative way to accomplish grasping without considering the force closure explicitly. This is done by applying some forces around the object using impedance controllers \cite{wimbock2007impedance, wimboeck2006passivity}, expecting that the forces automatically form a force closure. This paper also employs this strategy.

\subsubsection{CoM trajectory}
Employing the coordinate transformation $\bxi=\bm{T}\bV$, the control is designed using the CoM dynamics (recall (\ref{eq:new_dynamics})). This implies that the desired trajectory should be designed for the CoM, not the base frame. One way to resolve this issue is to compute the CoM trajectory from the AM's desired trajectory. Another way is to design the system to have a lightweight arm so that the CoM is very closely located at the geometric center of the floating base.

\section{Simulation Validation}
\label{sec:simulation}
To validate the proposed controller, the collaborative grasping is simulated in PyBullet environments with  $0.001\mathrm{s}$ time step. A 3 DOF non-redundant planar manipulator is attached to the quadrotor, as shown in Fig. \ref{fig:collaborative_grasping}.  The task variable is defined by $\by= [y_x,\; y_z, \; y_\theta]^T\in\mathbb{R}^3$, where $y_{x}$, $y_{z}$ represent the position, and $y_{\theta}$ represents the orientation of the end-effector with respect to $\{w\}$.

The AMs are initially positioned symmetrically away from the object whose mass is $1\mathrm{kg}$. The simulation scenario is set as follows.
\begin{itemize}
	\item ${t=0\mathrm{s} - 5\mathrm{s}}$: approaching to an object (free-flight)
	\item ${t=5\mathrm{s} - 10\mathrm{s}}$: grasping an object
	\item ${t=10\mathrm{s} - 20\mathrm{s}}$: lifting an object and hovering 
\end{itemize}
Although multiple AMs are used in the simulation, we provide plots for only one AM, because the resulting plots are almost the same for all AMs.

\begin{figure} [!t]
\centering
	\subfigure[]
	{\includegraphics[scale=0.49]{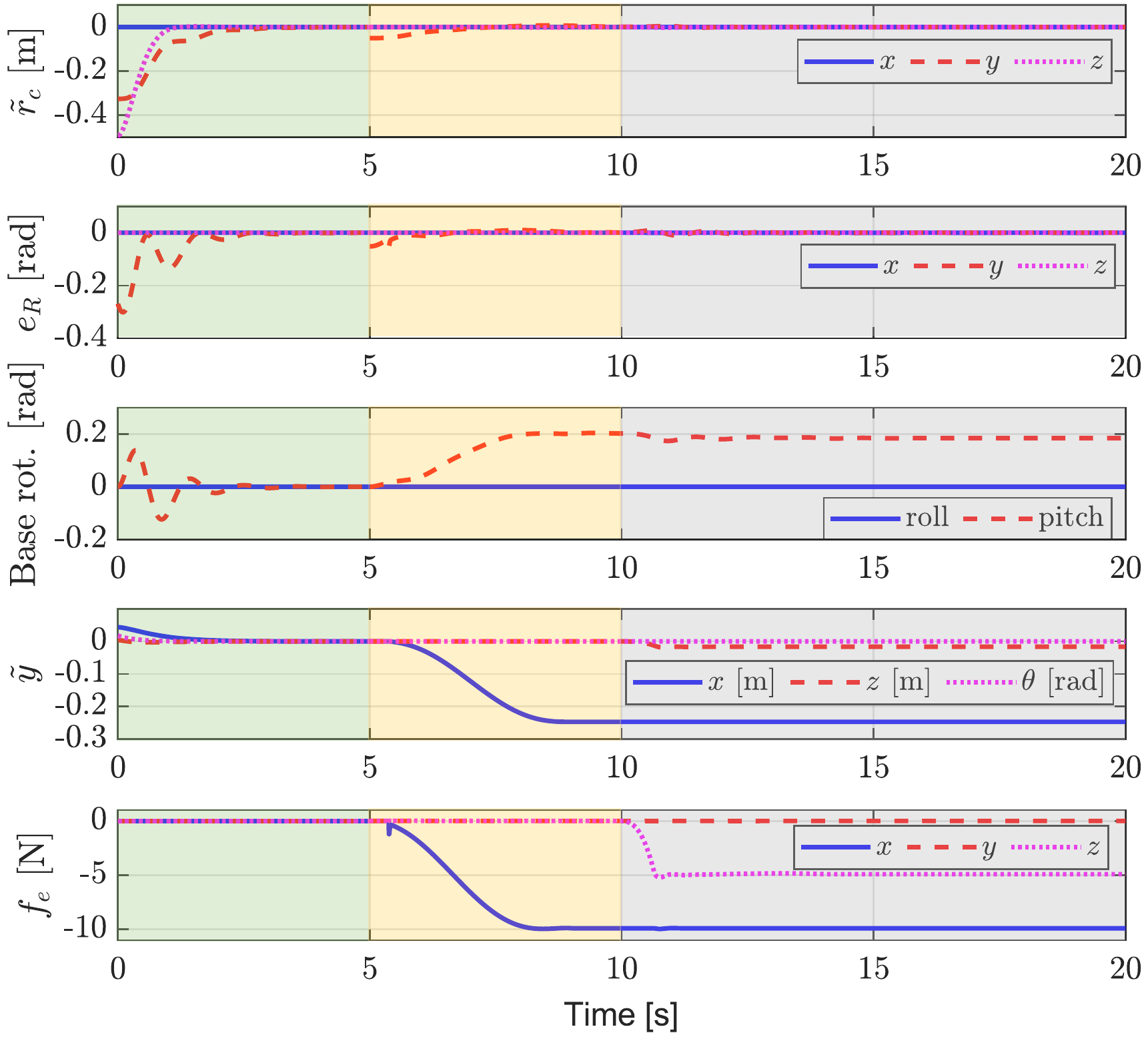}}
	\centering
	\subfigure[]
	{\includegraphics[scale=0.49]{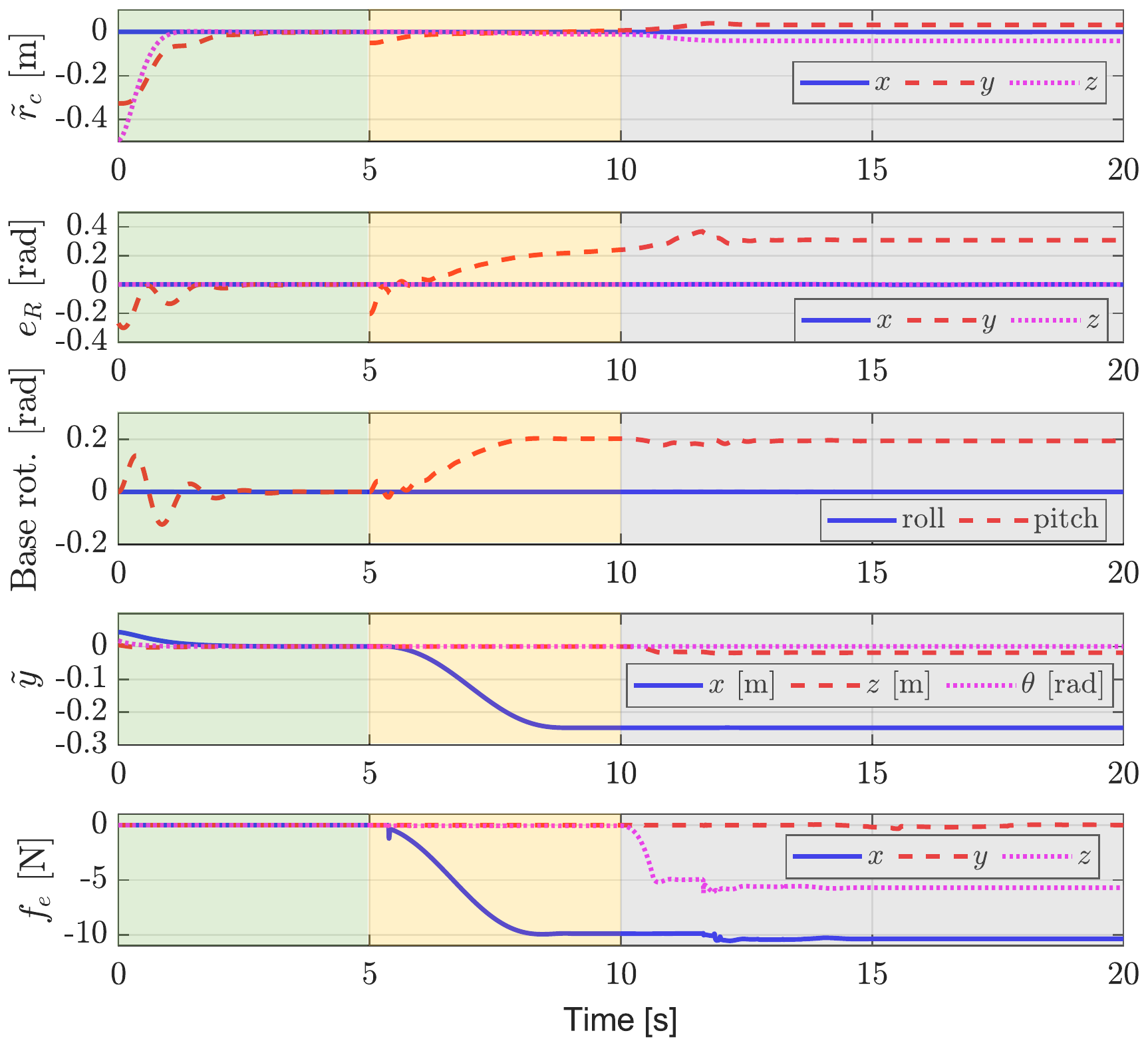}}
	\centering
	\caption{Simulation results for one AM. The green, yellow, and gray areas correspond to approaching, grasping, and hovering, respectively. (a) Collaborative grasping using two AMs with force compensation.
	(b) Collaborative grasping using two AMs without force compensation. 
	}
	\label{fig:simulation}
\end{figure}

\subsection{Collaborative grasping using two AMs with force compensation}
In the first simulation, two AMs are used to grasp the object  using the controller (\ref{eq:Com_controller}), (\ref{eq:ori_controller}), and (\ref{eq:u3}); see Fig. \ref{fig:collaborative_grasping}. The simulation results are shown in Fig. \ref{fig:simulation}(a). During the free-flight ($0\mathrm{s}\leq t\leq5\mathrm{s}$), the CoM position error, the floating base orientation error, and the task error converge to zero as shown in the first, second, and fourth rows of Fig. \ref{fig:simulation}(a). Therefore, asymptotic stability stated in Theorem \ref{thm:stability} is verified. 

During collaborative grasping ($5\mathrm{s}\leq t\leq10\mathrm{s}$), the external force $\bm{f}_e$ measured by the F/T sensor is shown in the fifth row of Fig. \ref{fig:simulation}(a). $-f_{e, x}$ and $-f_{e, z}$ imply the normal contact force and the tangential contact force, respectively. Since $y_{x, d}$ is defined inside the object, the normal contact force is generated. The tangential contact force is zero because the object is located on the table. 

After grasping, two aerial manipulators lift the object and hover at ${t\geq10\mathrm{s}}$. While hovering, slipping between the end-effector and the object does not occur because the impedance control generates the sufficiently large grasping force that causes appropriate friction force. Note that $f_{e,z}$ in the fifth row of Fig. \ref{fig:simulation}(a) is $-4.9\mathrm{N}$ which is half of the gravitational force on the object.
Another important observation is that the two quadrotors are tilted in the direction of the grasping force as shown in Fig. \ref{fig:collaborative_grasping}; see also the red dashed line in the third row of Fig. \ref{fig:simulation}(a). At this point, the CoM and the orientation error are zero because of the force compensation (recall, from (\ref{eq:Com_controller}) and (\ref{eq:Rd}), that the desired orientation is defined using the desired thrust force that contains the external force). On the other hand, since the $\by$-dynamics is a mass-spring-damper system with external force (see (\ref{eq:desired_task_behavior})), a constant error occurs for $\tilde{y}_x$. The foregoing observation verifies Theorem \ref{thm:stability_grasping}.

\subsection{Collaborative grasping using two AMs without force compensation}
In the second simulation, the external force compensation is not applied in the control law (\ref{eq:Com_controller}) and (\ref{eq:ori_controller}). The simulation results are shown in Fig. \ref{fig:simulation}(b). The free-flight phase is the same as the first simulation. When hovering ($t\geq10\mathrm{s}$), the CoM error and the floating base orientation error converge to a certain point, which is not zero as shown in the first two rows of Fig. \ref{fig:simulation}(b). Nevertheless, it should be pointed out that the AMs are successfully grasping the object. This verifies the discussion in Section \ref{sec:discussion_ft}.

\begin{figure} [!t]
\centering
	{\includegraphics[scale=0.43]{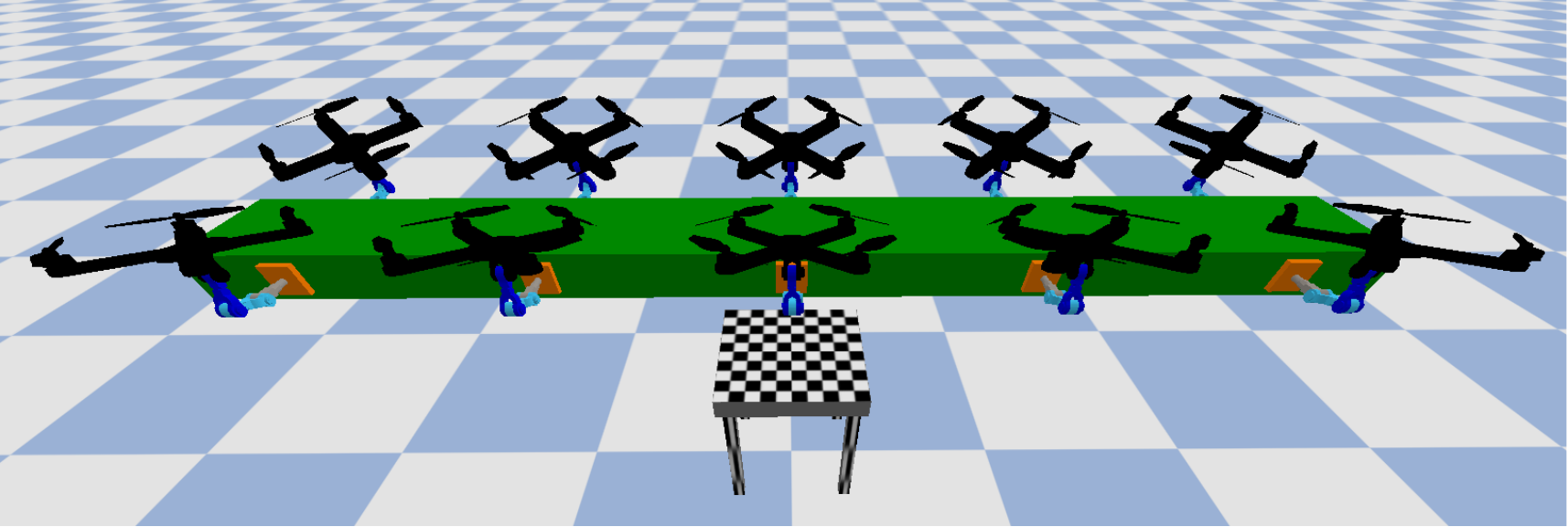}}
	\centering
	\caption{A snapshot of collaborative grasping using 10 AMs.
	}
	\label{fig:hovering}
\end{figure}

\subsection{Collaborative grasping using 10 AMs}
To validate the scalability of the proposed control scheme (recall Section \ref{sec:discussion_scalability}), a collaborative grasping is performed using 10 AMs, as shown in Fig. \ref{fig:hovering}. In this simulation, the object mass is $5\mathrm{kg}$. Although the resulting plot is omitted because it is very similar to the previous ones, the video attachment demonstrates this scenario.

\section{Conclusion}
\label{sec:conclusion}
This paper presents a passive impedance control scheme to accomplish collaborative grasping using under-actuated aerial manipulators (AMs). The proposed control scheme is decentralized in the sense that each AM can be controlled entirely using only its own states.  Modularity, which is gained by passivity (Theorem \ref{thm:passivity}), enables an arbitrary number of AMs to collaborate stably. During the collaborative grasping, each system automatically converges to a certain configuration at which AMs can balance the interaction forces (Theorem \ref{thm:stability_grasping}). Apart from grasping, uniform asymptotic stability is guaranteed for the free-flight phase (Theorem \ref{thm:stability}). Simulation confirms that the collaborative grasping can be accomplished using multiple AMs. 
\bibliographystyle{IEEEtran}
\bibliography{IEEEabrv,ICRA2023}

\begin{thebibliography}{10}
\providecommand{\url}[1]{#1}
\csname url@rmstyle\endcsname
\providecommand{\newblock}{\relax}
\providecommand{\bibinfo}[2]{#2}
\providecommand\BIBentrySTDinterwordspacing{\spaceskip=0pt\relax}
\providecommand\BIBentryALTinterwordstretchfactor{4}
\providecommand\BIBentryALTinterwordspacing{\spaceskip=\fontdimen2\font plus
\BIBentryALTinterwordstretchfactor\fontdimen3\font minus
  \fontdimen4\font\relax}
\providecommand\BIBforeignlanguage[2]{{%
\expandafter\ifx\csname l@#1\endcsname\relax
\typeout{** WARNING: IEEEtran.bst: No hyphenation pattern has been}%
\typeout{** loaded for the language `#1'. Using the pattern for}%
\typeout{** the default language instead.}%
\else
\language=\csname l@#1\endcsname
\fi
#2}}

\bibitem{trujillo2019novel}
M.~{\'A}. Trujillo, J.~R. Mart{\'\i}nez-de Dios, C.~Mart{\'\i}n, A.~Viguria,
  and A.~Ollero, ``Novel aerial manipulator for accurate and robust industrial
  ndt contact inspection: A new tool for the oil and gas inspection industry,''
  \emph{Sensors}, vol.~19, no.~6, p. 1305, 2019.

\bibitem{tognon2019truly}
M.~Tognon, H.~A.~T. Ch{\'a}vez, E.~Gasparin, Q.~Sabl{\'e}, D.~Bicego,
  A.~Mallet, M.~Lany, G.~Santi, B.~Revaz, J.~Cort{\'e}s, \emph{et~al.}, ``A
  truly-redundant aerial manipulator system with application to push-and-slide
  inspection in industrial plants,'' \emph{IEEE Robotics and Automation
  Letters}, vol.~4, no.~2, pp. 1846--1851, 2019.

\bibitem{bodie2020active}
K.~Bodie, M.~Brunner, M.~Pantic, S.~Walser, P.~Pf{\"a}ndler, U.~Angst,
  R.~Siegwart, and J.~Nieto, ``Active interaction force control for
  contact-based inspection with a fully actuated aerial vehicle,'' \emph{IEEE
  Transactions on Robotics}, vol.~37, no.~3, pp. 709--722, 2020.

\bibitem{bernard2011autonomous}
M.~Bernard, K.~Kondak, I.~Maza, and A.~Ollero, ``Autonomous transportation and
  deployment with aerial robots for search and rescue missions,'' \emph{Journal
  of Field Robotics}, vol.~28, no.~6, pp. 914--931, 2011.

\bibitem{bernard2010load}
M.~Bernard, K.~Kondak, and G.~Hommel, ``Load transportation system based on
  autonomous small size helicopters,'' \emph{The aeronautical journal}, vol.
  114, no. 1153, pp. 191--198, 2010.

\bibitem{korpela2012mm}
C.~M. Korpela, T.~W. Danko, and P.~Y. Oh, ``Mm-uav: Mobile manipulating
  unmanned aerial vehicle,'' \emph{Journal of Intelligent \& Robotic Systems},
  vol.~65, no.~1, pp. 93--101, 2012.

\bibitem{orsag2013modeling}
M.~Orsag, C.~Korpela, and P.~Oh, ``Modeling and control of mm-uav: Mobile
  manipulating unmanned aerial vehicle,'' \emph{Journal of Intelligent \&
  Robotic Systems}, vol.~69, no.~1, pp. 227--240, 2013.

\bibitem{fumagalli2016mechatronic}
M.~Fumagalli, S.~Stramigioli, and R.~Carloni, ``Mechatronic design of a robotic
  manipulator for unmanned aerial vehicles,'' in \emph{IEEE/RSJ International
  Conference on Intelligent Robots and Systems (IROS)}, 2016, pp. 4843--4848.

\bibitem{sarkisov2019development}
Y.~S. Sarkisov, M.~J. Kim, D.~Bicego, D.~Tsetserukou, C.~Ott, A.~Franchi, and
  K.~Kondak, ``Development of sam: cable-suspended aerial manipulator,'' in
  \emph{IEEE International Conference on Robotics and Automation (ICRA)}, 2019,
  pp. 5323--5329.

\bibitem{pounds2011yale}
P.~E. Pounds, D.~R. Bersak, and A.~M. Dollar, ``The yale aerial manipulator:
  grasping in flight,'' in \emph{IEEE International Conference on Robotics and
  Automation (ICRA)}, 2011, pp. 2974--2975.

\bibitem{suarez2018design}
A.~Suarez, G.~Heredia, and A.~Ollero, ``Design of an anthropomorphic,
  compliant, and lightweight dual arm for aerial manipulation,'' \emph{IEEE
  Access}, vol.~6, pp. 29\,173--29\,189, 2018.

\bibitem{bellicoso2015design}
C.~D. Bellicoso, L.~R. Buonocore, V.~Lippiello, and B.~Siciliano, ``Design,
  modeling and control of a 5-dof light-weight robot arm for aerial
  manipulation,'' in \emph{2015 23rd Mediterranean Conference on Control and
  Automation (MED)}.\hskip 1em plus 0.5em minus 0.4em\relax IEEE, 2015, pp.
  853--858.

\bibitem{forte2012impedance}
F.~Forte, R.~Naldi, A.~Macchelli, and L.~Marconi, ``Impedance control of an
  aerial manipulator,'' in \emph{2012 American Control Conference (ACC)}.\hskip
  1em plus 0.5em minus 0.4em\relax IEEE, 2012, pp. 3839--3844.

\bibitem{fumagalli2012modeling}
M.~Fumagalli, R.~Naldi, A.~Macchelli, R.~Carloni, S.~Stramigioli, and
  L.~Marconi, ``Modeling and control of a flying robot for contact
  inspection,'' in \emph{IEEE/RSJ International Conference on Intelligent
  Robots and Systems (IROS)}, 2012, pp. 3532--3537.

\bibitem{lippiello2012exploiting}
V.~Lippiello and F.~Ruggiero, ``Exploiting redundancy in cartesian impedance
  control of uavs equipped with a robotic arm,'' in \emph{IEEE/RSJ
  International Conference on Intelligent Robots and Systems (IROS)}, 2012, pp.
  3768--3773.

\bibitem{acosta2014robust}
J.~{\'A}. Acosta, M.~Sanchez, and A.~Ollero, ``Robust control of underactuated
  aerial manipulators via ida-pbc,'' in \emph{53Rd IEEE Conference on Decision
  and Control}, 2014, pp. 673--678.

\bibitem{kim2018passive}
M.~J. Kim, R.~Balachandran, M.~De~Stefano, K.~Kondak, and C.~Ott, ``Passive
  compliance control of aerial manipulators,'' in \emph{IEEE/RSJ International
  Conference on Intelligent Robots and Systems (IROS)}, 2018, pp. 4177--4184.

\bibitem{yuksel2019aerial}
B.~Y{\"u}ksel, C.~Secchi, H.~H. B{\"u}lthoff, and A.~Franchi, ``Aerial physical
  interaction via ida-pbc,'' \emph{The International Journal of Robotics
  Research}, vol.~38, no.~4, pp. 403--421, 2019.

\bibitem{yang2014dynamics}
H.~Yang and D.~Lee, ``Dynamics and control of quadrotor with robotic
  manipulator,'' in \emph{IEEE International Conference on Robotics and
  Automation (ICRA)}, 2014, pp. 5544--5549.

\bibitem{garofalo2018task}
G.~Garofalo, F.~Beck, and C.~Ott, ``Task-space tracking control for
  underactuated aerial manipulators,'' in \emph{2018 European Control
  Conference (ECC)}.\hskip 1em plus 0.5em minus 0.4em\relax IEEE, 2018, pp.
  628--634.

\bibitem{jimenez2022precise}
A.~Jim{\'e}nez-Cano, D.~Sanalitro, M.~Tognon, A.~Franchi, and J.~Cort{\'e}s,
  ``Precise cable-suspended pick-and-place with an aerial multi-robot system,''
  \emph{Journal of Intelligent \& Robotic Systems}, vol. 105, no.~3, pp. 1--13,
  2022.

\bibitem{tagliabue2019robust}
A.~Tagliabue, M.~Kamel, R.~Siegwart, and J.~Nieto, ``Robust collaborative
  object transportation using multiple mavs,'' \emph{The International Journal
  of Robotics Research}, vol.~38, no.~9, pp. 1020--1044, 2019.

\bibitem{mohammadi2016cooperative}
M.~Mohammadi, A.~Franchi, D.~Barcelli, and D.~Prattichizzo, ``Cooperative
  aerial tele-manipulation with haptic feedback,'' in \emph{IEEE/RSJ
  International Conference on Intelligent Robots and Systems (IROS)}, 2016, pp.
  5092--5098.

\bibitem{yang2015hierarchical}
H.~Yang and D.~Lee, ``Hierarchical cooperative control framework of multiple
  quadrotor-manipulator systems,'' in \emph{IEEE International Conference on
  Robotics and Automation (ICRA)}, 2015, pp. 4656--4662.

\bibitem{caccavale2015cooperative}
F.~Caccavale, G.~Giglio, G.~Muscio, and F.~Pierri, ``Cooperative impedance
  control for multiple uavs with a robotic arm,'' in \emph{IEEE/RSJ
  International Conference on Intelligent Robots and Systems (IROS)}, 2015, pp.
  2366--2371.

\bibitem{zelazo2015decentralized}
D.~Zelazo, A.~Franchi, H.~H. B{\"u}lthoff, and P.~Robuffo~Giordano,
  ``Decentralized rigidity maintenance control with range measurements for
  multi-robot systems,'' \emph{The International Journal of Robotics Research},
  vol.~34, no.~1, pp. 105--128, 2015.

\bibitem{kim2019passivity}
M.~J. Kim, W.~Lee, J.~Y. Choi, G.~Chung, K.-L. Han, I.~S. Choi, C.~Ott, and
  W.~K. Chung, ``A passivity-based nonlinear admittance control with
  application to powered upper-limb control under unknown environmental
  interactions,'' \emph{IEEE/ASME Transactions on Mechatronics}, vol.~24,
  no.~4, pp. 1473--1484, 2019.

\bibitem{kim2021passive}
M.~J. Kim, A.~Werner, F.~Loeffl, and C.~Ott, ``Passive impedance control of
  robots with viscoelastic joints via inner-loop torque control,'' \emph{IEEE
  Transactions on Robotics}, vol.~38, no.~1, pp. 584--598, 2021.

\bibitem{ott2008cartesian}
C.~Ott, \emph{Cartesian impedance control of redundant and flexible-joint
  robots}.\hskip 1em plus 0.5em minus 0.4em\relax Springer, 2008.

\bibitem{khalil2002nonlinear}
H.~K. Khalil and J.~Grizzle, \emph{Nonlinear systems}.\hskip 1em plus 0.5em
  minus 0.4em\relax Prentice hall Upper Saddle River, 2002, vol.~3.

\bibitem{lee2010geometric}
T.~Lee, M.~Leok, and N.~H. McClamroch, ``Geometric tracking control of a
  quadrotor uav on se (3),'' in \emph{49th IEEE Conference on Decision and
  Control}, 2010, pp. 5420--5425.

\bibitem{slotine1991applied}
J.-J.~E. Slotine, W.~Li, \emph{et~al.}, \emph{Applied nonlinear control}.\hskip
  1em plus 0.5em minus 0.4em\relax Prentice hall Englewood Cliffs, NJ, 1991,
  vol. 199, no.~1.

\bibitem{murray2017mathematical}
R.~M. Murray, Z.~Li, and S.~S. Sastry, \emph{A mathematical introduction to
  robotic manipulation}.\hskip 1em plus 0.5em minus 0.4em\relax CRC press,
  2017.

\bibitem{wimbock2007impedance}
T.~Wimbock, C.~Ott, and G.~Hirzinger, ``Impedance behaviors for two-handed
  manipulation: Design and experiments,'' in \emph{IEEE International
  Conference on Robotics and Automation (ICRA)}, 2007, pp. 4182--4189.

\bibitem{wimboeck2006passivity}
T.~Wimboeck, C.~Ott, and G.~Hirzinger, ``Passivity-based object-level impedance
  control for a multifingered hand,'' in \emph{IEEE/RSJ International
  Conference on Intelligent Robots and Systems (IROS)}, 2006, pp. 4621--4627.

\end{thebibliography}

\end{document}